\documentclass[letterpaper]{article}
\usepackage{proceed2e}
\usepackage[margin=1in]{geometry}
\usepackage{times}
\usepackage{helvet}
\usepackage{courier}
\usepackage[citefix,noobjects,nofunctions]{miritools}
\usepackage{natbib}

\usepackage{amsmath} 
\usepackage{amsthm} 
\usepackage{amssymb}  
\usepackage{amsopn}

\usepackage{pgfplots}
\usepackage{critchabbreviations}

\usepackage{multirow}

\usetikzlibrary{arrows}
\usetikzlibrary{positioning}
\usepackage{color}
\definecolor{envNodeColor}{rgb}{0.6,0.6,0.9}
\definecolor{envArrowColor}{rgb}{0.2,0.2,0.8}
\definecolor{intNodeColor}{rgb}{0.5,0.8,0.5}
\definecolor{intArrowColor}{rgb}{0.1,0.8,0.1}
\definecolor{boolNodeColor}{rgb}{0.8,0.5,0.5}
\definecolor{boolArrowColor}{rgb}{0.8,0.2,0.2}

\newcommand{\seq}{\bar}

\newcommand{\si}[1]{^{#1}}

\newcommand{\Do}{\mathrm{Do}}
\newcommand{\hist}[1]{h_{#1}}
\newcommand{\althist}[1]{h_{#1}'}
\newcommand{\dohist}[1]{o_{\le #1} \mid \Do(a_{<#1})}

\newcommand{\val}[4]{E^{#1}_{#2}(#3; #4)}

\newcommand{\bool}{B}

\newcommand{\fig}[1]{Figure~\ref{fig:#1}}
\newcommand{\eqn}[1]{Equation~\ref{eqn:#1}}
\newcommand{\lem}[1]{Lemma~\ref{lem:#1}}
\newcommand{\thm}[1]{Theorem~\ref{thm:#1}}

\newcommand{\prop}[1]{Proposition~\ref{prop:#1}}
\newcommand{\cor}[1]{Corollary~\ref{cor:#1}}

\DeclareMathOperator*{\argmax}{argmax}
\DeclareMathOperator{\image}{image}

\newtheorem{proposition}[theorem]{Proposition}
\newtheorem{corollary}[theorem]{Corollary}

\title{Servant of Many Masters: \\ Shifting priorities in Pareto-optimal sequential decision-making}

\author{Andrew Critch, Stuart Russell \\ University of California, Berkeley \\ \{critch, russell\}@berkeley.edu}

\begin{document}

\maketitle

\begin{abstract}
It is often argued that an agent making decisions on behalf of two or more principals who have different utility functions should adopt a {\em Pareto-optimal} policy, i.e., a policy that cannot be improved upon for one agent without making sacrifices for another.  A famous theorem of Harsanyi shows that, when the principals have a common prior on the outcome distributions of all policies, a Pareto-optimal policy for the agent is one that maximizes a fixed, weighted linear combination of the principals' utilities.   
 
In this paper, we show that Harsanyi's theorem does not hold for principals with different priors, and derive a more precise generalization which does hold, which constitutes our main result.  In this more general case, the relative weight given to each principal's utility should evolve over time according to how well the agent's observations conform with that principal's prior. The result has implications for the design of contracts, treaties, joint ventures, and robots.


\end{abstract}


\section{Introduction}

As AI systems take on an increasingly pivotal decision-making role in human society, an important question arises: \emph{Whose values should a powerful decision-making machine be built to serve?} \citep{bostrom2014superintelligence}

Consider, informally, a scenario wherein two or more principals---perhaps individuals, companies, or states---are considering cooperating to build or otherwise obtain an ``agent" that will then interact with an environment on their behalf.  The ``agent'' here could be anything that follows a policy, such as a robot, a corporation, or a web-based AI system. In such a scenario, the principals will be concerned with the question of ``how much'' the agent will prioritize each principal's interests, a question which this paper addresses quantitatively.

One might be tempted to model the agent as maximizing the expected value, given its observations, of some utility function $U$ of the environment that equals a weighted sum 
\begin{equation}\label{eqn:harsanyi}
w\si{1}U\si{1} + w\si{2}U\si{2}
\end{equation}
of the principals' individual utility functions $U\si{1}$ and $U\si{2}$, as Harsanyi's social aggregation theorem \citep{harsanyi1980cardinal} recommends.  Then the question of prioritization could be reduced to that of choosing values for the weights $w\si{i}$.  

However, this turns out to be a suboptimal approach, from the perspective of the principals.  As we shall see in \prop{impossibility}, this solution form is not generally compatible with Pareto-optimality when agents have different beliefs.  Harsanyi's setting does not account for agents having different priors, nor for decisions being made sequentially, after future observations.  

In such a setting, we need a new form of solution, exhibited in this paper.  The solution is presented along with a recursion (\thm{main}) that characterizes solutions by a process algebraically similar to, but meaningfully different from, Bayesian updating.  The updating process resembles a kind of bet-settling between the principals, which allows them each to expect to benefit from the veracity of their own beliefs.

Qualitatively, this phenomenon can be seen in isolation whenever two people make a bet on a piece of decision-irrelevant trivia.  If neither Alice nor Bob would base any important decision on whether Michael Jackson was born in 1958 or 1959, they might still make a bet for \$100 on the answer.  For a person chosen to arbitrate the bet (their ``agent"), Michael Jackson's birth year now becomes a decision-relevant observation: it determines which of Alice and Bob gets the money!

Even in scenarios where differences in belief are not decision-irrelevant, once might expect some ``degree" of bet-settling to arise from the disagreement.  The main result of this paper (\thm{main}) is a precise formulation of exactly how and how much a Pareto-optimal agent will tend to prioritize each of its principals over time, as a result of differences in their implicit predictions about the agent's observations.

\subsection{Related work}

This paper may be viewed as extending or complimenting results in several areas:

\paragraph{Value alignment theory.}  The ``single principal'' value alignment problem---that of aligning the value function of an agent with the values of single human, or a team of humans in close agreement with one another---is already a very difficult one and should not be swept under the rug; approaches like inverse reinforcement learning (IRL) \citep{russell1998learning} \citep{ng2000algorithms} \citep{abbeel2004apprenticeship} and cooperative inverse reinforcement learning (CIRL) \citep{hadfield2016cooperative} have only begun to address it.

\paragraph{Social choice theory.} The whole of social choice theory and voting theory may be viewed as an attempt to specify an agreeable formal policy to enact on behalf of a group.  Harsanyi's utility aggregation theorem \citep{harsanyi1980cardinal} suggests one form of solution: maximizing a linear combination of group members' utility functions.  The present work shows that this solution is inappropriate when principals have different beliefs, and \thm{main} may be viewed as an extension of Harsanyi's form that accounts simultaneously for differing priors and the prospect of future observations.  Indeed, Harsanyi's form follows as a direct corollary of \thm{main} when principals do share the same beliefs (\cor{harsanyi}).

\paragraph{Bargaining theory.} The formal theory of bargaining, as pioneered by \citep{nash1950bargaining} and carried on by \citep{myerson1979incentive}, \citep{myerson2013game}, and \citep{myerson1983efficient}, is also topical.  Future investigation in this area might be aimed at generalizing their work to sequential decision-making settings, and this author recommends a focus on research specifically targeted at resolving conflicts.

\paragraph{Multi-agent systems.} There is ample literature examining multi-agent systems using sequential decision-making models.  \citet{shoham2008multiagent} survey various models of multiplayer games using an MDP to model each agent's objectives.  Chapter 9 of the same text surveys social choice theory, but does not account for sequential decision-making.  

\citet{zhang2014fairness} may be considered a sequential decision-making approach to social choice: they use MDPs to represent the decisions of players in a competitive game, and exhibit an algorithm for the players that, if followed, arrives at a Pareto-optimal Nash equilibrium satisfying a certain fairness criterion.  Among the literature surveyed here, that paper is the closest to the present work in terms of its intended application: roughly speaking, achieving mutually desirable outcomes via sequential decision-making.  However, that work is concerned with an ongoing interaction between the players, rather than selecting a policy for a single agent to follow as in this paper.  


\paragraph{Multi-objective sequential decision-making.} There is also a good deal of work on Multi-Objective Optimization (MOO)  \citep{tzeng2011multiple}, including for sequential decision-making, where solution methods have been called Multi-Objective Reinforcement Learning (MORL).  For instance, \citet{gabor1998multi} introduce a MORL method called Pareto Q-learning for learning a set of a Pareto-optimal polices for a Multi-Objective MDP (MOMDP).  \citet{soh2011evolving} define Multi-Reward Partially Observable Markov Decision Processes (MR-POMDPs), and use use genetic algorithms to produce non-dominated sets of policies for them.  \citet{roijers2015point} refer to the same problems as Multi-objective POMDPS (MOPOMDPs), and provide a bounded approximation method for the optimal solution set for all possible weightings of the objectives.  \citet{wang2014multi} surveys MORL methods, and contributes Multi-Objective Monte-Carlo Tree Search (MOMCTS) for discovering multiple Pareto-optimal solutions to a multi-objective optimization problem.   \citet{wray2015multi} introduce Lexicographic Partially Observable Markov Decision Process (LPOMDPs), along with two accompanying solution methods.

However, none of these or related works addresses scenarios where the objectives are derived from principals with differing beliefs, from which the priority-shifting phenomenon of \thm{main} arises.  Differing beliefs are likely to play a key role in negotiations, so for that purpose, the formulation of multi-objective decision-making adopted here is preferable.

\section{Notation}

Random variables are denoted by uppercase letters, e.g., $S_1$, and lowercase letters, e.g., $s_1$, are used as indices ranging over the values of a variable, as in the equation
\[
\EE[S_1] = \sum_{s_1} \PP(s_1)\cdot s_1.
\]

Given a set $A$, the set of probability distributions on $A$ is denoted $\Delta A$.

Sequences are denoted by overbars, e.g., given a sequence $(s_1,\ldots,s_n)$, $\seq s$ stands for the whole sequence.   Subsequences are denoted
by subscripted inequalities, so e.g., $s_{\le 4}$ stands for $(s_1,s_2,s_3)$,
and $s_{> 4}$ stands for $(s_5,\ldots,s_n)$.

\section{Formalism}\label{sec:formalism}
\emph{N.B.: All results in this paper generalize directly from agents with two principals to agents with several, but for clarity of exposition, the case of two principals will be prioritized.}

Consider a scenario wherein Alice and Bob will share some cake, and have different predictions of the cake's color.  Even if the color would be decision-irrelevant for either Alice or Bob on their own (they don't care what color the cake is), we will show that the difference between their predictions will tend to make the cake color a decision-relevant observation for a Pareto-optimal cake-splitting policy that is adopted before they see the cake.  Specifically, we will show that Pareto-optimal policies tend to incorporate some degree of bet-settling between Alice and Bob, where the person who was more right about the color of the cake will end up getting more of it.

\subsection{Serving multiple principals as a single POMDP}
To formalize such scenarios, where a single agent acts on behalf of multiple principals, we need some definitions.

We encode each principal $j$'s view of the agent's decision problem as a finite horizon POMDP, $D\si{j} = (\Ss\si{j},\Aa,T\si{j}, U\si{j}, \Oo, \Omega\si{j}, n)$, which simultaneously represents that principal's beliefs about the environment, and the principal's utility function (see \citet{russell2003artificial} for an introduction to POMDPs). These symbols take on their usual meaning:
\begin{itemize}
\item $\Ss\si{j}$ represents a set of possible states $s$ of the environment,
\item $\Aa$ represents the set of possible actions $a$ available to the agent,
\item $T\si{j}$ represents the conditional probabilities principal $j$ believes will govern the environment state transitions, i.e., $\PP\si{j}(s_{i+1}\mid s_i a_i)$,
\item $U\si{j}$ represents principal $j$'s utility function from sequences of environmental states $(s_1,\ldots,s_n)$ to $\RR$; for the sake of generality, $U\si{j}$ is \emph{not assumed} to be additive over time, as reward functions often are, 
\item $\Oo$ represents the set of possible observations $o$ of the agent, 
\item $\Omega\si{j}$ represents the conditional probabilities principal $j$ believes will govern the agent's observations, i.e., $\PP\si{j}(o_i\mid s_i)$, and
\item $n$ is the horizon (number of time steps)
\end{itemize}

This POMDP structure is depicted by the Bayesian network in \fig{pomdp}.  (See \citet{darwiche2009modeling} for an intro to Bayesian networks.)  At each point in time $i$, the agent has a time-specific policy $\pi_i$, which receives the agent's \emph{history},
\[
h_i := (o_{\le i},a_{<i}),
\]
and returns a distribution $\pi_i(- \mid \hist i)$ on actions $a_i$, which will then be used to generate an action $a_i$ with probability $\pi(a_i \mid \hist i)$.  Thus, principal $j$'s subjective probability of an outcome $(\seq s, \seq o, \seq a)$ is given by a probability distribution $\PP\si{j}$ that takes $\pi$ as a parameter:
\begin{multline}
\label{eqn:pomdp}
\PP\si{j}(\seq s, \seq o, \seq a ; \pi) := \PP\si{j}(s_1) \cdot \prod_{i=1}^{n} 
\PP\si{j}(o_i \mid s_i)\, \\  \pi(a_i \mid \hist i)\, \PP\si{j}(s_{i+1} \mid s_i,a_i)
\end{multline}

\textbf{Full-memory assumption.}  \emph{Every policy $\pi$ in this paper will be assumed to employ a ``full memory"}, so it decomposes into a sequence of policies $\pi_i$ for each time step.  In \fig{pomdp}, the part of the Bayes net governed by the full-memory policy is highlighted in green.

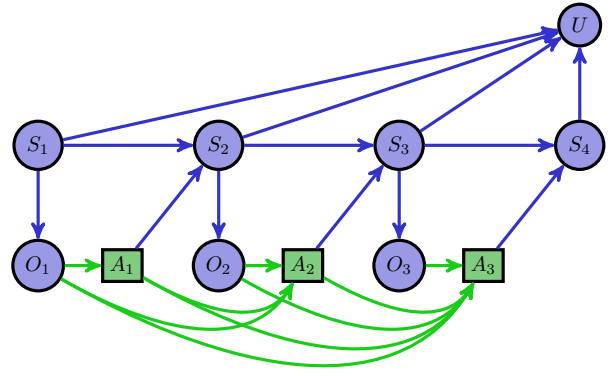
\begin{figure}
\pgfmathtruncatemacro{\len}{2}
\pgfmathtruncatemacro{\lenn}{\len + 1}
\pgfmathtruncatemacro{\lennn}{\len + 2}
\begin{center}
\begin{tikzpicture}[scale=0.8, every node/.style={transform shape}, node distance=1cm, auto, very thick, >=stealth']

  \foreach \x in {1,...,\lennn}{
        \node [draw, circle, fill=envNodeColor]  (s\x) at (3*\x,2) {$S_\x$};
        \node [color=envArrowColor, below left = -0.1cm and 0.3cm of s\x] (T\x) {};
        }

  \foreach \x in {1,...,\lenn}{
       \node [draw, circle, fill=envNodeColor]  (o\x) at (3*\x,0) {$O_\x$};
       \node [draw, rectangle, fill=intNodeColor]  (a\x) at (3*\x+1.4,0) {$A_\x$};
       }

  \foreach \x in {1,...,\lenn}{
      \pgfmathtruncatemacro{\xn}{\x +1 }
      \draw[->,envArrowColor] (s\x) to node [right] {} (o\x);
      \draw[->,intArrowColor] (o\x) to node {} (a\x);
      \draw[->,envArrowColor] (a\x) to node {} (s\xn);
      \draw[->,envArrowColor] (s\x)--(s\xn);
      }
      
  \foreach \x in {1,...,\len}{
      \pgfmathtruncatemacro{\xn}{\x +1 };
      \foreach \y in {\xn,...,\lenn}{
          \draw[->,intArrowColor] (o\x) to [out=330,in=240] node {} (a\y);
          \draw[->,intArrowColor] (a\x) to [out=330,in=240] node {} (a\y);
          }
      }
  
  \node [draw, circle, fill=envNodeColor] (u) at (3*\lennn,4) {$U$};

  \foreach \x in {1,...,\lennn}{
          \draw[->,envArrowColor] (s\x) to node {} (u);  
  } 
\end{tikzpicture}
\end{center}
\vspace{-6ex}
\caption{A POMDP with horizon $n=\lenn$ (in blue), being solved by a full-memory policy (in green).}
\label{fig:pomdp}
\end{figure}
\paragraph{Common knowledge assumptions.}  It is assumed that the principals will have common knowledge of the (full-memory) policy $\pi = (\pi_1,\ldots,\pi_n)$ they select for the agent to implement, but that the principals may have different beliefs about how the environment works, and of course different utility functions.  It is also assumed that the principals have common knowledge of one another's current beliefs at the time of the agent's creation, which we refer to as their their priors.

\emph{This last assumption is critical.}  During the agent's creation, one should expect each principal's beliefs to have updated somewhat in response to disagreements from the other.   Assuming common knowledge of their priors means assuming the principals to have reached an equilibrium where, each knowing what the other believes, they do not wish to further update their own beliefs.\footnote{It is enough to assume the principals have reached a ``persistent disagreement'' that cannot be mediated by the agent in some way.  Future work should design solutions for facilitating the process of attaining common knowledge, or to obviate the need to assume it.}

\subsection{Pareto-optimal policies}

A policy will be considered Pareto-optimal relative to a set of POMDPs it could be deployed to solve.

\begin{definition}[Compatible POMDPs]  We say that two POMDPs, $D\si{1}$ and $D\si{2}$, are \emph{compatible} if any policy for one may be viewed as a policy for the other, i.e., they have the same set of actions $\Aa$ and observations $\Oo$, and the same number of time steps $n$.
\end{definition}

In this context, where a single policy $\pi$ may be evaluated relative to more than one POMDP, we use superscripts to represent which POMDP is governing the probabilities and expectations, e.g.,
\[
\EE\si{j}[U\si{j}; \pi] := \sum_{\seq s \in (\Ss\si{j})^n} \PP\si{j}(\bar s; \pi) U\si{j}(\bar s)
\]
represents the expectation in $D\si{j}$ of the utility function $U\si{j}$, assuming policy $\pi$ is followed.
\begin{definition}[Pareto-optimal policies]  A policy $\pi$ is \emph{Pareto-optimal} for a set of compatible POMDPs $(D\si{1},\ldots,D\si{k})$ if for any other policy $\pi'$ and any $j\in\{1,\ldots,k\}$
\[
\EE\si{j}[U\si{j};\pi'] > \EE\si{j}[U\si{j};\pi] \Rightarrow (\exists \ell) \left(\EE\si{\ell}[U\si{\ell};\pi'] < \EE\si{\ell}[U\si{\ell};\pi]\right),
\]
\end{definition}
It is assumed that, before the agent's creation, the principals will be seeking a Pareto-optimal (full-memory) policy for the agent to follow, relative to the POMDPs $D\si{j}$ describing each principal's view of the agent's task.

\subsection{Example: cake betting}
A quantitative model of a cake betting scenario is laid out in Table \ref{table:scenario}, and described as follows.  

Alice (Principal 1) and Bob (Principal 2) are about to be presented with a cake which they can choose to split in half to share, or give entirely to one of them.  They have (built or purchased) a robot that will make the cake-splitting decision on their behalf.  Alice's utility function returns $0$ if she gets no cake, $20$ if she gets half a cake, or $30$ if she gets a whole cake.  Bob's utility function values Bob getting cake in the same way.

\begin{table*}[t]
\centering
\begin{tabular}{*{6}{|c}|}
\hline
$S_1= O_1$& $\PP\si{1}(O_1)$ & $\PP\si{2}(O_1)$ & $A_1=S_2$ & $U\si{1}$ & $U\si{2}$ \\ \hline
\multirow{ 3}{*}{red cake}
& \multirow{ 3}{*}{$90\%$}
& \multirow{ 3}{*}{$10\%$}
& (all, none) & 30 & 0 \\ 
&&& (half, half) & 20 & 20 \\ 
&&& (none, all) & 0 & 30 \\ 
\cline{1-6}
\multirow{ 3}{*}{green cake}
& \multirow{ 3}{*}{$10\%$}
& \multirow{ 3}{*}{$90\%$}
& (all, none) & 30 & 0 \\ 
&&& (half, half) & 20 & 20 \\ 
&&& (none, all) & 0 & 30 \\ 
\hline
\end{tabular}
\caption{
An example scenario wherein a Pareto-optimal full-memory policy undergoes priority shifting (who gets the cake), based on features that are decision-irrelevant for each principal (cake color).
}
\label{table:scenario}
\end{table*}

However, Alice and Bob have different beliefs about the color of the cake.  Alice is $90\%$ sure that the cake is red ($S_1=O_1=\text{``red"}$), versus $10\%$ sure it will be green ($S_1=O_1=\text{``green"}$), whereas Bob's probabilities are reversed.  

Upon seeing the cake, the robot must decide to either give Alice the entire cake ($A_1=S_2=\text{(all, none)}$), split the cake half-and-half ($A_1=S_2=\text{(half, half)}$), or give Bob the entire cake ($A_1=S_2=\text{(none, all)}$).  Moreover, Alice and Bob have common knowledge of all these facts.

Now, consider the following Pareto-optimal full-memory policy that favors Alice (Principal 1) when $O_1$ is red, and Bob (Principal 2) when $O_1$ is green:
\begin{align*}
&\hat\pi(- \mid \text{red}) = 100\%\text{(all, none)}\\
&\hat\pi(- \mid \text{green}) = 100\%\text{(none, all)}
\end{align*}
This policy can be viewed intuitively as a bet between Alice and Bob about the value of $O_1$, and is highly appealing to both principals:
\begin{align*}
\EE\si{1}[U\si{1}; \hat\pi] &= 90\%(30) + 10\%(0) = 27\\
\EE\si{2}[U\si{2}; \hat\pi] &= 10\%(0) + 90\%(30) = 27
\end{align*}
In particular, $\hat\pi$ is more appealing to both Alice and Bob than an agreement to deterministically split the cake (half, half), which would yield them each an expected utility of $20$.  However, 

\begin{proposition}\label{prop:impossibility}
The Pareto-optimal strategy $\hat\pi$ above cannot be implemented by any agent that na\"{i}vely maximizes a fixed-over-time linear combination of the conditionally expected utilities of the two principals.  That is, it cannot be implemented by any policy $\pi$ satisfying 
\begin{multline}\label{eqn:naive}
\pi(- \mid o_1) \in \argmax_{\alpha\in\Delta A}\left(r\cdot\EE\si{1}[U\si{1} \mid o_1; a_1\sim\alpha] + (1-r)\right. \\ 
\left. \cdot\EE\si{2}[U\si{2} \mid o_1; a_1\sim\alpha]\right)
\end{multline}
for some fixed $r\in[0,1]$.  Moreover, every such policy $\pi$ is strictly worse than $\hat\pi$ in expectation to one of the principals.
\end{proposition} 

\begin{proof}
See appendix.
\end{proof}

This proposition is relatively unsurprising when one considers the full-memory policy $\hat\pi$ intuitively as a bet-settling mechanism, because the nature of betting is to favor different preferences based on future observations.  However, to be sure of this impossibility claim, one must rule out the possibility that the $\hat\pi$ could be implemented by having the agent choose which element of the $\argmax$ in \eqn{naive} to use based on whether the cake appears red or green.  (See appendix.)

\subsection{Characterizing Pareto-optimality geometrically}
With the definitions above, we can characterize a Pareto-optimality as a geometric condition.
\paragraph{Policy mixing assumption.}  Given policies $\pi\si{1},\ldots,\pi\si{R}$ and a distribution $\alpha=(\alpha\si{1},\ldots,\alpha\si{R})\in\Delta\{1,\ldots,R\}$, we assume that the agent may construct a new policy by choosing at time 0 between the $\pi\si{r}$ with probability $\alpha\si{r}$, and then executing the chosen policy for the rest of time.  We write this policy as
$
\pi = \sum_r \alpha\si{r}\pi\si{r},
$
whence we derive:
\begin{equation}\label{eqn:linearity}
\EE\si{j}\left[U\si{j} ; \sum_r \alpha\si{r} \pi\si{r}\right] = \sum_r \alpha\si{r} \EE\si{j}[U\si{j} ; \pi\si{r}].
\end{equation}

\begin{lemma}[Polytope Lemma]\label{lem:pareto}
A full-memory policy $\pi$ is Pareto-optimal to principals $1$ and $2$ if and only if there exist weights $w\si{1},w\si{2}\geq 0$ with $w\si{1}+w\si{2}=1$ such that
\begin{equation}\label{eqn:pareto}
\pi \in \argmax_{\pi^* \in\Pi}\left(w\si{1}\EE\si{1}[U\si{1};\pi^*] + w\si{2}\EE\si{2}[U\si{2}; \pi^*]\right)
\end{equation}
\end{lemma}

\begin{proof}
The mixing assumption gives the set of policies $\Pi$ the structure of a convex space that the maps $\EE\si{j}[U\si{j}; -] $ respect by \eqn{linearity}.  This ensures that the image of the map $f:\Pi\to\RR\si{2}$ given by
\[
f(\pi) := \left(\EE\si{1}[U\si{1};\pi],\; \EE\si{2}[U\si{2};\pi]\right)
\]
is a closed, convex polytope.  As such, a point $(x,y)$ lies on the Pareto boundary of $\image(f)$ if and only if there exist nonnegative weights $(w\si{1},w\si{2})$, not both zero, such that 
\[
(x,y) \in \argmax_{(x^*,y^*)\in \image(f)} \left(w\si{1}x^* + w\si{2}y^*\right)
\]
After normalizing $w\si{1}+w\si{2}$ to equal $1$, this implies the result.
\end{proof}

\subsection{Characterizing Pareto-optimality probabilistically}  
To help us apply the Polytope Lemma, we will adopt an interpretation wherein the weights $w\si{i}$ are subjective probabilities for the agent, as follows.

For any $w\in\Delta\{1,2\}$, we define a new POMDP, $D$, that works by flipping a $(w\si{1},w\si{2})$-weighted coin, and then running $D\si{1}$ or $D\si{2}$ thereafter, according to the coin flip.  We denote this by 
\[
D=w\si{1}D\si{1} + w\si{2}D\si{2},
\] and call $D$ a \emph{POMDP mixture}.
A formal definition of $D$ is given in the appendix.  It can be depicted by a Bayes net by adding an additional environmental node for $\bool $ in the diagram of $D\si{1}$ and $D\si{2}$ (see \fig{pomdp2}).

\begin{figure}
\pgfmathtruncatemacro{\len}{2}
\pgfmathtruncatemacro{\lenn}{\len + 1}
\pgfmathtruncatemacro{\lennn}{\len + 2}
\begin{center}
\begin{tikzpicture}[scale=0.7, every node/.style={transform shape}, node distance=1cm, auto, very thick, >=stealth']
        \node [draw, circle, fill=boolNodeColor]  (b) at (3,4) {$\bool $};
        
  \foreach \x in {1,...,\lennn}{
        \node [draw, circle, fill=envNodeColor]  (s\x) at (3*\x,2) {$S_\x$};
        \draw [->,boolArrowColor] (b) to node {} (s\x);
        \node [color=envArrowColor, below left = -0.1cm and 0.3cm of s\x] (T\x) {};
        }

  \foreach \x in {1,...,\lenn}{
       \node [draw, circle, fill=envNodeColor]  (o\x) at (3*\x,0) {$O_\x$};
       \draw [->,boolArrowColor] (b) to [out=270,in=145] node [right] {} (o\x);
       \node [draw, rectangle, fill=intNodeColor]  (a\x) at (3*\x+1.4,0) {$A_\x$};
       }

  \foreach \x in {1,...,\lenn}{
      \pgfmathtruncatemacro{\xn}{\x +1 }
      \draw[->,envArrowColor] (s\x) to node [right] {} (o\x);
      \draw[->,intArrowColor] (o\x) to node {} (a\x);
      \draw[->,envArrowColor] (a\x) to node {} (s\xn);
      \draw[->,envArrowColor] (s\x)--(s\xn);
      }
      
  \foreach \x in {1,...,\len}{
      \pgfmathtruncatemacro{\xn}{\x +1 };
      \foreach \y in {\xn,...,\lenn}{
          \draw[->,intArrowColor] (o\x) to [out=330,in=240] node {} (a\y);
          \draw[->,intArrowColor] (a\x) to [out=330,in=240] node {} (a\y);
          }
      }
  
  \node [draw, circle, fill=envNodeColor] (u) at (3*\lennn,4) {$U$};
  \draw[->,boolArrowColor] (b) to node [right] {} (u);

  \foreach \x in {1,...,\lennn}{
          \draw[->,envArrowColor] (s\x) to node {} (u);  
  }
\end{tikzpicture}
\end{center}
\vspace{-6ex}
\caption{A POMDP (mixture) with horizon $n=\lenn$ initialized by a Boolean $\bool $, being solved by a full-memory policy (green)}
\label{fig:pomdp2}
\end{figure}
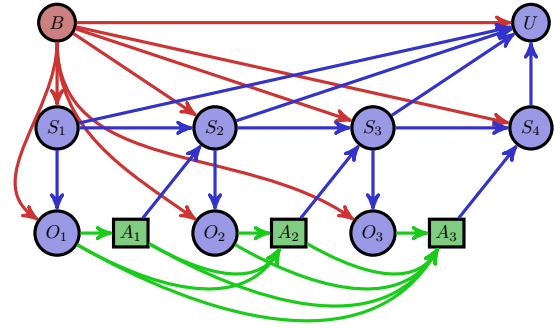

Given any full-memory policy $\pi$, the expected payoff of $\pi$ in $w\si{1}D\si{1}+w\si{2}D\si{2}$ is exactly
\begin{align*}
&\; \PP(\bool =1)\cdot\EE[U \mid \bool =1 ; \pi] \\  
+ &\; \PP(\bool =2)\cdot\EE[U \mid \bool =2 ; \pi]\\
= &\; w\si{1}\EE\si{2}[U\si{1}; \pi] + w\si{2}\EE\si{2}[U\si{2} ; \pi]
\end{align*}

Therefore, using the above definitions, \lem{pareto} may be restated in the following equivalent form:

\begin{lemma}[Mixture Lemma]\label{lem:mixture}
Given a pair $(D\si{1},D\si{2})$ of compatible POMDPs, a full-memory policy $\pi$ is Pareto-optimal for that pair if and only if there exists $w\in\Delta\{1,2\}$ such that $\pi$ is an optimal full-memory policy for the single POMDP given by $w\si{1}D\si{1}+w\si{2}D\si{2}$.  
\end{lemma}

Expressed in the form of \eqn{pareto}, it might not be clear how a Pareto-optimal full-memory policy makes use of its observations over time, aside from storing them in memory.  For example, is there any sense in which the agent carries ``beliefs" about the environment that it ``updates" at each time step?  \lem{mixture} allows us to reduce some such questions about Pareto-optimal policies to questions about single POMDPs.  

If $\pi$ is an optimal full-memory policy for a single POMDP, the optimality of each action distribution $\pi_i(-\mid\hist{i})$ can be characterized without reference to the previous policy components 
$(\pi_1,\ldots,\pi_{i-1})$, nor to $\pi_i(-\mid\althist{i})$ for any alternate history $\althist{i}$.  This can be expressed using Pearl's ``$\Do()$" notation \citep{pearl2009causality}:
\begin{definition}[``do'' notation] The probability of $\seq o$ \emph{causally conditioned on $\seq a$} is defined as
\begin{multline*}
\PP\si{j}(\seq o \mid \Do(\seq a)) \\ := \sum_{\seq s \in (\Ss\si{j})^n} \PP\si{j}(s_1) \cdot \prod_{i=1}^n
\PP\si{j}(o_i \mid s_i)\, \PP\si{j}(s_{i+1} \mid s_ia_i)
\end{multline*}
\end{definition}
\begin{definition}[Expected utility abbreviation]
For brevity, given any POMDP $D$ and policy $\pi$, we write
\[
\val{D}{\pi}{\alpha}{h_i} := \EE[U \mid \hist{i} ;\; a_n\sim \alpha; \; \pi_{>i}].
\]
i.e., the total expected utility in $D$ that would result from replacing $\pi_i(-\mid h_i)$ by $\alpha$.  This quantity does not depend on $\pi_{\le i}$.
\end{definition}
\begin{proposition}[Classical separability]\label{prop:separability}
If $D$ is a POMDP described by conditional probabilities $\PP(-\mid -)$ and utility function $U$ (as in \eqn{pomdp}), then a full-memory policy $\pi$ is optimal for $D$ if and only if for each time step $i$ and each observation/action history $\hist{i}$, the action distribution $\pi_i(-\mid \hist{i})$ satisfies the following backward recursion:
\begin{multline*}
 \pi_i(-\mid\hist{i}) \in \argmax_{\alpha\in\Delta A}\left( \PP\left(\dohist{i}\right)\cdot \val{D}{\pi}{\alpha}{h_i}\right)
\end{multline*}
This characterization of $\pi_i(-|\hist{i})$ does not refer to $\pi_1,\ldots,\pi_{i-1}$, nor to 
$\pi_i(\althist{i})$ for any alternate history $\althist{i}$.
\end{proposition}

\begin{proof}
This is just Bellman's Principle of Optimality.  See  \citep{bellman1957dynamic}, Chap. III. 3.
\end{proof}

\emph{N.B.: Unlike Bellman's ``backup" equation, the above proposition requires no assumption whatsoever on the form of the utility function.  Note also that when the probability term $\PP(\dohist{i})$ is non-zero, it may be removed from the $\argmax$ without changing the theorem statement.  But when the term is zero, its presence is essential, and implies that $\pi_i(-\mid\hist{i})$ can be anything.}

It turns out that Pareto-optimality can be characterized in a similar way by backward recursion from the final time step.  The resulting recursion reveals a pattern in how the weights on the principals' conditionally expected utilities must change over time, which is the main result of this paper:

\begin{theorem}[Pareto-optimal control theorem]\label{thm:main}
Given a pair $(D\si{1},D\si{2})$ of compatible POMDPs with horizon $n$, a full-memory policy $\pi$ is Pareto-optimal if and only if its components $\pi_i$ for $i\le n$ satisfy the following backward recursion for some weights $w\in\Delta\{1,2\}$:
\begin{align*}
\pi\si{i}&(-\mid\hist{i}) \in \argmax_{\alpha\in\Delta A} \biggl(\vphantom{\left(\left(\right)\right)}\biggr.\\
 &\; w\si{1} \PP\si{1}\left(\dohist{i}\right) \cdot \val{D^1}{\pi}{\alpha}{h_i}\\
+ &\; \biggl.w\si{2} \PP\si{2}\left(\dohist{i}\right) \cdot \val{D^2}{\pi}{\alpha}{h_i} \biggr)
\end{align*}

In words, to achieve Pareto-optimality, the agent must
\begin{enumerate}
\item use each principal's own world-model $D^j$ when estimating the degree $\val{D^j}{\pi}{\alpha}{h_i}$ to which a decision $\alpha$ favors that principal's utility function, and
\item shift the relative priority of each principal's expected utility in the agent's maximizationtarget over time, by a factor proportional to how well that principal's prior predicts the agent's observations, $\PP\si{i}\left(\dohist{i}\right)$.
\end{enumerate}
\end{theorem}

\emph{N.B.: The analogous result for more than two POMDPs holds as well, with essentially the same proof.}

\begin{proof}[Proof of \thm{main}]
By \lem{mixture}, the Pareto-optimality of $\pi$ for $(D\si{1},D\si{2})$ is equivalent to its classical optimality for
$D=w\si{1}D\si{1} + w\si{2}D\si{2}$ for some $(w\si{1},w\si{2})$.
Writing $\PP$ for probabilities in D, \prop{separability} says this is equivalent to $\alpha = \pi\si{i}(-\mid\hist{i})$ maximizing the following expression $F(\alpha)$ for each $i$:
\begin{equation}\label{eqn:score}
F(\alpha)=\PP\left(\dohist{i}\right)\cdot \val{D}{\pi}{\alpha}{\hist i}.
\end{equation}
The expectation factor on the right equals
\begin{align*}
\val{D}{\pi}{\alpha}{\hist i} 
= &\; \PP\left(B=1\mid o_{\le i}, \Do(a_{<i})\right)\cdot\val{D\si{1}}{\pi}{\alpha}{\hist i}\\
+ &\; \PP\left(B=2\mid o_{\le i}, \Do(a_{<i})\right)\cdot\val{D\si{2}}{\pi}{\alpha}{\hist i}.
 \end{align*}
Multiplying by 
\begin{align*}
\PP\left(\dohist{i}\right) =&\; w\si{1} \PP\si{1}\left(\dohist{i}\right)\\
 +&\; w\si{2} \PP\si{2}\left(\dohist{i}\right) 
\end{align*}
and applying Bayes' rule yields that 
\begin{align*}
F(\alpha) = &\; w\si{1} \PP\si{1}\left(\dohist{i}\right)\val{D\si{1}}{\pi}{\alpha}{\hist i}\\
+ &\; w\si{2} \PP\si{2}\left(\dohist{i}\right) \val{D\si{2}}{\pi}{\alpha}{\hist i},
\end{align*}
hence the result.
\end{proof}

To see the necessity of the $\PP\si{j}$ terms that shift the expectation weights in \thm{main} over time, recall from \prop{impossibility} that, without these, some Pareto-optimal policies cannot be implemented.  These $\PP\si{j}$ terms are responsible for the ``bet-settling" phenomena discussed in the introduction.

However, when the principals have the same beliefs, they aways assign the same probability to the agent's observations, so the weights on their respective valuations do not change over time.  Hence, as a special instance, we derive:

\begin{corollary}[Harsanyi's utility aggregation formula]\label{cor:harsanyi}
Suppose that principals 1 and 2 share the same beliefs about the environment, i.e., the pair $(D\si{1},D\si{2})$ of compatible POMDPs agree on all parameters except the principals' utility functions $U\si{1}\neq U\si{2}$. Then a full-memory policy $\pi$ is Pareto-optimal if and only if there exists $w\in\Delta\{1,2\}$ such that for $i\le n$, $\pi_i$ satisfies
\begin{multline*}
\pi\si{i}(-\mid\hist{i}) \in \argmax_{\alpha\in\Delta A} \left(\right. \\
\left. \EE[w\si{1}U\si{1}+w\si{2}U\si{2}] \mid \hist{i} ;\; a_i\sim \alpha;\; \pi_{>i}]\right)
\end{multline*}
where $\EE=\EE\si{1}=\EE\si{2}$ denotes the shared expectations of both principals.
\end{corollary}
\begin{proof}
Setting $\EE=\EE\si{1}=\EE\si{2}$ in \thm{main}, factoring out the common coefficient $\PP\si{1}\left(\dohist{i}\right)=\PP\si{2}\left(\dohist{i}\right)$, and applying linearity of expectation yields the result.
\end{proof}

\section{Conclusion}

\thm{main} exhibits a novel form for the objective of a sequential decision-making policy that is Pareto-optimal according to principals with differing beliefs.  

This form represents two departures from na\"{i}ve utility aggregation: to achieve Pareto-optimality for principals with differing beliefs, an agent must (1) use each principal's own beliefs (updated on the agent's observations) when evaluating how well an action will serve that principal's utility function, and (2) shift the relative priority it assigns to each principal's expected utilities over time, by a factor proportional to how well that principal's prior predicts the agent's observations.

\subsection{Implications for contract design}

\thm{main} has implications for modeling and structuring the process of contract design.  If a contract is being created between principals with different beliefs, then to the extent that the principals will target Pareto-optimality among them as an objective, there will be a tendency for the contract to end up implicitly settling bets between the principals.  Perhaps making the bet-settling nature of Pareto-optimal contract design more explicit might help to design contracts that are more attractive to both principals, along the lines illustrated by \prop{impossibility}.  This could potentially lead to more successful negotiations, provided the principals remained willing to uphold the contract after its implicit bets have been settled.

\subsection{Implications for shareable AI systems}

\prop{impossibility} shows how the Pareto-optimal form of \thm{main} is more attractive---from the perspective of the principals---than policies that do not account for differences in their beliefs.  The relative attractiveness of shared ownership versus individual ownership of AI systems may be essential to the technological adoption of shared systems.  Consider the following product substitutions that might be enabled by the development of shareable machine learning systems:

\begin{itemize}

\item Office assistant software jointly controlled by a team, as an improvement over personal assistant software for each member of the team.

\item A team of domestic robots controlled by a family, as an improvement over individual robots each controlled by a separate family member.

\item A web-based security system shared by several interested companies or nations, as an improvement over individual security systems deployed by each group.

\end{itemize}

It may represent a significant technical challenge for any of these substitutions to become viable.  However, machine learning systems that are able to approximate Pareto-optimality as an objective are more likely to be sufficiently appealing to motivate the switch from individual control to sharing.

\subsection{Implications for bargaining versus racing}

Consider two nations---allies or adversaries---who must decide whether to cooperate in the deployment of a very powerful and autonomous AI system. 

If the nations cannot reach agreement as to what policy a jointly owned AI system should follow, joint ownership may be less attractive than building separate AI systems, one for each party.  This could lead to an arms race between nations competing under time pressure to develop ever more powerful militarized AI systems. Under such race conditions, everyone loses, as each nation is afforded less time to ensure the safety and value alignment of its own system. 

The first author's primary motivation for this paper is to initiate a research program with the mission of averting such scenarios.  Beginning work today on AI architectures that are more amenable to joint ownership could help lead to futures wherein powerful entities are more likely to share and less likely to compete for the ownership of such systems. 

\subsection{Future work}

Insofar as \thm{main} is not particularly mathematically sophisticated---it employs only basic facts about convexity and linear algebra---this suggests there may be more low-hanging fruit to be found in the domain of ``machine implementable social choice theory".   Future work should address methods for helping the principals to share information---perhaps in exchange for adjustments to the weights in \thm{main}---to reach either a state of agreement or a persistent disagreement that allows the theorem to be applied.  More ambitiously, bargaining models that account for a degree of transparency between the principals should be employed, as individual humans and institutions have some capacity for detecting one another's intentions.

As well, scenarios where the principals continue to exhibit some active control over the system after its creation should be modeled in detail.  In real life, principals usually continue to exist in their agents' environments, and accounting for this will be a separate technical challenge.

As a final motivating remark, consider that social choice theory and bargaining theory were both pioneered during the Cold War, when it was particularly compelling to understand the potential for cooperation between human institutions that might behave competitively.  In the coming decades, machine intelligence will likely bring many new challenges for cooperation, as well as new means to cooperate, and new reasons to do so.  As such, new technical aspects of social choice and bargaining will likely continue to emerge.

\section{Appendix}

Here we make available the technical details for defining POMDP mixtures, and proving that certain Pareto-optimal expectations cannot be obtained without priority-shifting.

\begin{definition}[POMDP mixtures]\label{defn:mixture} Suppose that $D\si{1}$ and $D\si{2}$ are compatible POMDPs, with parameters 
$D\si{j} = (\Ss\si{j},\Aa,T\si{j}, U\si{j}, \Oo, \Omega\si{j}, n)$.  Define a new POMDP compatible with both, denoted $D=w\si{1}D\si{1} + w\si{2}D\si{2}$, with parameters 
$D\si{j} = (\Ss,\Aa,T, U, \Oo, \Omega, n)$, as follows:
\begin{itemize}
\item $\Ss:= \{(j,s) \mid j\in\{1,2\}, s\in\Ss\si{j}\}$,
\item Environmental transition probabilities $T$ given by
\begin{align*}
& \PP\left((j,s_1)\right) := \: w\si{j}\cdot \PP\si{j}(s_1)\\
\intertext{for any initial state $s_1\in\Ss\si{j}$, and thereafter,}
& \PP\left((j',s_{i+1}) \mid (j,s_i), a_i \right) & \\
& := \;
\begin{cases}
\PP\si{j}\left(s_{i+1} \mid s_ia_i\right) &\mbox{ if $j'=j$}\\
0 & \mbox { if $j'\neq j$}
\end{cases}
\end{align*}

Hence, the value of $j$ will be constant over time, so a full history for the environment may be represented by a pair
\[
(j,\seq s) \in \{1\}\times (\Ss\si{1})^n \cup \{2\}\times(\Ss\si{2})^n.
\]
Let $\bool $ denote the boolean random variable that equals whichever constant value of $j$ obtains, so then
\[
\PP(\bool =j) = w\si{j}
\]
\item The utility function $U$ is given by
\[
U(j,\seq s) := U\si{j}(\seq s)
\]
\item The observation probabilities $\Omega$ are given by
\[
\PP\left(o_i \mid (j,s_i)\right) := \PP(\bool =j) \cdot \PP\si{j}(o_i \mid s_i)
\]
In particular, the agent does not observe directly whether $j=1$ or $j=2$.
\end{itemize}

\end{definition}
\begin{proof}[\textbf{Proof of \prop{impossibility}}] Suppose $\pi$ is any policy satisfying \eqn{naive} for some fixed $r$, and consider the following cases for $r$:
\begin{enumerate}
\item If $r < 1/3 $, then $\pi$ must satisfy
\[
\pi(-\mid o_1) = 100\%\text{(none, all)}.
\] 
Here, $\EE\si{1}[U\si{1} ; \pi] = 0 < 27$, so $\pi$ is strictly worse than $\hat\pi$ in expectation to Alice.

\item If $r = 1/3 $, then $\pi$ must satisfy
\[
\pi(-\mid o_1) = q(o_1)\text{(none, all)} + (1-q(o_1))\text{(half, half)}
\] 
for some $q(o_1)\in[0,1]$ depending on $o_1$.  Here, $\EE\si{1}[U\si{1} ; \pi] \le 20 < 27$ (with equality when $q(\text{red})=q(\text{green})=1$), so $\pi$ is strictly worse than $\hat\pi$ in expectation to Alice.

\item If $1/3 < r < 2/3 $, then $\pi$ must satisfy
\[
\pi(-\mid o_1)=100\%\text{(half, half)}
\]
Here, $\EE\si{1}[U\si{1} ; \pi] = \EE\si{2}[U\si{2} ; \pi] = 20 < 27$, so $\pi$ is strictly worse than $\hat\pi$ in expectation to both Alice and Bob.
\end{enumerate}
The remaining cases, $r=2/3$ and  $r>2/3$, are symmetric to the first two, with Bob in place of Alice and (none, all) in place of (all, none).

Hence, no fixed linear combination of the principals' utility functions can be maximized to simultaneously achieve an expected utility of 27 for both players.
\end{proof}









\bibliographystyle{named}
\bibliography{main}

\end{document}